\documentclass{article}
\usepackage{ijcai25}

\usepackage{times}
\usepackage{soul}
\usepackage{url}
\usepackage[hidelinks]{hyperref}
\usepackage[utf8]{inputenc}
\usepackage[small]{caption}
\usepackage{graphicx}
\usepackage{amsmath}
\usepackage{amsthm}
\usepackage{booktabs}
\usepackage{algorithm}
\usepackage{algorithmic}
\usepackage[switch]{lineno}
\usepackage{amssymb}
\usepackage{calrsfs}
\usepackage{pifont} 
\newtheorem{theorem}{Theorem}
\newtheorem{definition}{Definition}
\newtheorem{example}{Example}
\newtheorem{lemma}{Lemma}
\usepackage{xcolor}

\title{Polynomial-Time Relational Probabilistic Inference in Open Universes}

\author{
Luise Ge$^1$
\and
Brendan Juba$^1$\and
Kris Nilsson$^{1}$
\affiliations
$^1$Washington University in St. Louis\\
\emails
\{g.luise,bjuba,k.b.nilsson\}@wustl.edu,
}

\begin{document}

\maketitle

\begin{abstract}

Reasoning under uncertainty is a fundamental challenge in Artificial Intelligence. As with most of these challenges, there is a harsh dilemma between the expressive power of the language used, and the tractability of the computational problem posed by reasoning. Inspired by human reasoning, we introduce a method of first-order relational probabilistic inference that satisfies both criteria, and can handle hybrid (discrete and continuous) variables. Specifically, we extend sum-of-squares logic of expectation to relational settings, demonstrating that lifted reasoning in the bounded-degree fragment for knowledge bases of bounded quantifier rank can be performed in polynomial time, even with an \textit{a priori} unknown and/or countably infinite set of objects. Crucially, our notion of tractability is framed in \textit{proof-theoretic terms}, which extends beyond the syntactic properties of the language or queries.  We are able to derive the tightest bounds provable by proofs of a given degree and size and establish completeness in our sum-of-squares refutations for fixed degrees.

\end{abstract}

\section{Introduction}
Intelligent agents must cope with limits to their knowledge of the world. One of the most widely studied and mature approaches treats the state of the world as probabilistic, using inference from observed evidence to draw conclusions about what is likely—and what is not—regarding unobserved parts of the world, future states, and so on. 

At the same time, it is profitable to use relational representations for our knowledge of the world. In this way, a single, short statement can assert a common collection of relationships about a large collection of similar objects.
 In a vast world, such relational knowledge enables an agent to draw on some basic understanding of situations, locations, or objects that it has never before encountered. 
 There is substantial evidence from cognitive science showing that a core aspect of human intelligence arises from relational thinking and reasoning \cite{hummel2003symbolic,krawczyk2012cognition}.
 And one can argue that relational representations are highly desirable and perhaps even essential for many areas of AI such as planning and natural language processing.
 


 Unfortunately, power rarely comes for free. Na\"{\i}ve approaches to both probabilistic inference and relational inference, that might explicitly consider the various possible states of the world, are not computationally feasible to perform at scale. Since our focus is on addressing a vast world, we need approaches that do not explicitly consider the possible world states, but instead operate at a more abstract ``lifted'' level, reasoning generically about entire classes of objects. In general, it is still challenging to carry out such a strategy efficiently, and this is an area of active research \cite{van2021introduction}. In this work, we consider a new logic with a tractable fragment that can capture knowledge that is out of reach of the existing tractable methods.

Our approach builds on and extends three distinct threads in the literature. The first is work on tractable probability logics. Our work presents a relational generalization of the sum-of-squares probability logic \cite{lasserre01,juba2019polynomial}, while retaining a powerful tractable fragment given by bounded-degree expressions. 

The second thread is relational optimization. Sum-of-squares, in the primal formulation, is a family of semidefinite optimization problems. Kersting \textit{et al.}~\shortcite{kmt17} and Mladenov \textit{et al.}~\shortcite{mladenov2017lifted} respectively proposed relational generalizations of linear and convex quadratic programs, and developed tools for efficiently solving such problems. Here, we show that the techniques may be extended to countably infinite domains, and present a logic that characterizes the power of the programs we obtain via these techniques. 

Finally, the third thread concerns the existing technique for inference in such infinite universes: Belle \shortcite{belle2017open} showed how to perform weighted first-order model counting in such open universes, which enables certain kinds of probabilistic inference. 
We extend Belle's approach, and overcome two main limitations: first, that first-order model counting is only known to be tractable for various two-variable fragments, 
and is postulated to be intractable for many three-variable formulas and beyond \cite{beame2015symmetric}. Second, Belle's technique to extend to infinite universes relies on compactness for Boolean propositional logic, which clearly does not hold even for linear arithmetic---consider the example of the formulas $e(x)\geq n$ for all integer $n$ (where $e(x)$ represents the expected value for $x$, which is included in our language), the infinite set of constraints $\{e(x)\ge n: n \ \text{ is an integer}\}$ is unsatisfiable, whereas any finite subset is satisfiable. 

Our generalization, in contrast to others, enables us to reason about the expected values of random variables that take numeric values as well as bounded-degree moments. 
\section{Related Work}
\textbf{Probability logics} Most directly, our work proposes a logic of expectation, as discussed by Halpern and Pucella~\shortcite{halpern2007characterizing}, extended to a first-order language with a powerful tractable fragment. It directly generalizes the tractable fragment considered by Juba~\shortcite{juba2019polynomial}. Indeed, it can also serve as a logic of probability, of the kind that Halpern \shortcite{halpern1990analysis} calls a ``type 2'' first-order logic of probability -- meaning, the probabilities are taken over valuations of the formulas, as opposed to bindings from the domain of quantification. (See Appendix~\ref{appen:expecation} for an example)

Alternatively, our knowledge is expressed as a combination of logical constraints and bounds on the marginals of certain expressions; as shown by Ku{\v{z}}elka \textit{et al.}~\shortcite{kuzelka2018relational}, if we consider a maximum-entropy objective, marginal equality constraints alone capture Markov Logic Networks \cite{richardson2006markov}, so we may obtain a (strict) generalization of such models. 
But rather than assuming interest in a maximum-entropy model, we will be primarily interested in what values may possibly be taken by the expected value of other expressions.

\noindent \textbf{Markov Logic} Indeed, Markov Logic is a prominent example of the vast family of different models that have been considered as targets for probabilistic modeling and inference with relational representations; Van den Broeck et al.~\shortcite{van2021introduction} give a broad overview of this area. 
Typically, inference in these models is reduced to Weighted (First-Order) Model Counting (WFOMC) \cite{cd08,van2011lifted,van2014skolemization}, in which possible worlds are assigned a (pseudo-) likelihood given by the product of the (explicitly defined) weights associated with the specified formulas satisfied in each such possible world. Indeed, the likelihoods assigned to possible worlds by Markov Logic (for example again) are usually defined in such a way. Again, in contrast to these works, our focus is not specifically on inference in models of such form. 

Moreover, we consider expectation more generally than probability per se -- i.e., expressions that may take numeric values. 
In the literature, such models are referred to as ``hybrid'' probabilistic models, and reduced to Weighted (First-Order) Model Integration \cite{belle2015probabilistic,feldstein2021lifted}. This problem is generally more challenging and is only tractable in limited cases \cite{zeng2020probabilistic,feldstein2021lifted}.
Indeed, even WFOMC, which only concerns probability, is only known to be tractable in various two-variable fragments \cite{van2011lifted,van2014skolemization,beame2015symmetric,kazemi2016new,kazemi2017domain,kuusisto2018weighted,kuzelka2021weighted,van2023lifted,toth2023lifted}, and is believed to be intractable for some three-variable fragments \cite{beame2015symmetric}. 

Although our tractable fragment similarly requires some kind of bound on the quantifier rank, each finite bound corresponds to a polynomial-time fragment, and there is no such hard barrier against the many-variable fragment.

As mentioned previously, Belle's work \shortcite{belle2017open} extends these WFOMC-based approaches to infinite and open universe settings; but, being based on WFOMC, it also inherits the above barriers to tractability. 

\noindent\textbf{Other approaches} In addition, a number of other languages have been proposed for reasoning about probability in open universes, but without promises of tractability or generalization to reasoning about expectation \cite{poole2003first,milch2005blog,carbonetto2005nonparametric,poole2008independent}.

\section{Background and Notation}
By reasoning, we refer to the ability to resolve a query based on a knowledge base, both of which are encoded using the logic fragment we have defined. In this section, we first introduce the first-order probabilistic relational logic that underpins our inference method. We then discuss how sum-of-squares relaxation techniques can be applied to facilitate efficient reasoning at the propositional level, leaving the discussion of lifted reasoning for the following section.

\subsection{(First-Order) Probability Relational Logic}

\noindent\textbf{Language:} We start by defining a first-order language $\mathcal{L}$ that includes relational symbols $\{P(x), Q(x,y), R(x,y,z), ..., P'(x),...\}$ of every arity, variables $\mathcal{V}=\{x,y,z,...\}$, and a countably infinite set of names $\mathcal{N}$ serving as the domain of quantification. 
These names can be thought of as the integers in $\mathbb{N}$, but we  use proper names $\{john, james, jane,...\}$ for readability.
 We will also, in general, consider a finite set of constants $\mathcal{C}\subseteq\mathcal{N}$. and refer to $\mathcal{G}=\mathcal{N}\setminus\mathcal{C}$ as generic names. 
A renaming substitution is given by a permutation on $\mathcal{G}$, that is extended to $\mathcal{N}$ by the identity map on $\mathcal{C}$. 
The logical terms of the logic $\mathcal{T}$ are given by a relation symbol together with a tuple of the corresponding arity from $\mathcal{V}\cup \mathcal{C}$.
 A \textit{ground} term is given by a relation symbol together with a tuple of the corresponding arity from $\mathcal{N}$. 

In general, we consider expressions of the form of polynomial inequalities, where the indeterminates are given by terms: a monomial $\mu$ is given by a finite subset of $\mathcal{T},$ $\tau_1\ldots,\tau_k$ with corresponding positive integer exponents $d_1\ldots,d_k$: $\mu=\tau_1^{d_1}\cdots\tau_k^{d_k}$. We refer to $\deg(\mu)=\sum_{i=1}^kd_i$ as the degree of the monomial. 
For each monomial $\mu$, we have a moment term $e(\mu)$.

A polynomial inequality is now given by a finite set of monomials $M$, together with a real-valued coefficient $c_\mu$ for each monomial $\mu\in M$, and a numeric relation symbol from $\{\geq,=\}$. We will write these expressions as $\sum_{\mu\in M}c_\mu \mu \geq 0$ and $\sum_{\mu\in M}c_\mu \mu = 0$, respectively. We now refer to $\max_{\mu\in M}\deg(\mu)$ as the degree of the polynomial inequality expression.
Now, our polynomial inequalities will in general be bound by universal quantifiers, and analogous to the languages introduced by Lakemeyer and Levesque \shortcite{lakemeyer2002evaluation}, we will allow the domain of quantification to be specified by the following kind of equality expressions: these have atomic formulas consisting of a pair from $\mathcal{V}\cup\mathcal{C}$ (variables or constants), which we write e.g.\ as $x=y$, $jane = x$, etc., that are given by any Boolean expression on these atomic formulas (using the usual De Morgan connectives, $\land,\lor,\neg$).

\noindent\textbf{Logical Constraints} Now, for each pair of an equality expression $\Xi$ and polynomial inequality $\Lambda$, we have a logical constraint formula $\Phi=\forall \Xi\supset\Lambda$. When $\Xi$ is a trivial tautology, we simply write $\forall\Lambda$. We refer to the number of distinct variables occurring in $\Xi$ and $\Lambda$ together as the quantifier rank of $\Phi$. The semantics of $\Phi$ is that for all substitutions of names for variables $\theta$ satisfying $\Xi$, the same substitution into the polynomial inequality $\Lambda\theta$ is also satisfied.

\begin{example}
Although in general the relation symbols are interpreted as numeric indeterminates, we can assert that relations $P$ and $Q$ take Boolean values by the formulas $\forall P(x)^2-P(x)=0$ and $\forall Q(x,y)^2-Q(x,y)=0$. We can moreover define a negative literal for each relation symbol by considering another relation $\bar{P}$ and relating its value to $P$ like so: $\forall P(x)+\bar{P}(x)-1=0$. 

There are two ways of encoding clauses. The first approach uses a monomial of Boolean terms to represent a conjunction, and uses an equality formula to assert that the conjunction is false; by De Morgan's law, this is a clause. Concretely, the clause $\forall x,y P(x)\lor \neg Q(x,y)$ is represented by $\forall \bar{P}(x)Q(x,y) = 0$. This method produces monomials with degree equal to the width of the clause.  The second approach uses a linear inequality, asserting at least one term is true: we represent $\forall x,y P(x)\lor \neg Q(x,y)$ by 
$\forall P(x)+\bar{Q}(x,y)-1\geq 0$. These expressions have degree one. The two different encodings can be used to simulate two different tractable logics.

Finally, the equality formulas allow us to write expressions such as $\forall x\neq y\land x\neq jane\supset P(x)+\bar{Q}(x,y)-1\geq 0$, which is the aforementioned clause with additional restrictions on the possible bindings. As with Belle \shortcite{belle2017open}, we can thereby represent proper+ knowledge bases \cite{lakemeyer2002evaluation}, which consist of universal clauses with such equality expressions defining the domain of quantification.
\end{example}

\noindent \textbf{Expectation Constraints} Next, we use these relational symbols, and their associated variables, as the random variables in our expectation logic in the sense of Halpern and Pucella \shortcite{halpern2007characterizing}. The term $e(P(james,jane))$ represents the expected value of the relation $P$ on the names of $james$ and $jane$. 
We now consider expectation bound expressions that are defined by a finite set of moment terms $M$, together with real-valued coefficients $c_{e(\mu)}$ for each moment $e(\mu)\in M$, corresponding to the linear inequality $\sum_{e(\mu)\in M}c_{e(\mu)}e(\mu)\geq 0$. The degree of the expectation bound is similarly defined to be $\max_{e(\mu)\in M}\deg(\mu)$, where we also refer to the degree of the monomial $\mu$ as the degree of the moment $e(\mu)$. Similarly to the logical constraint formulas, for each pair of an equality expression $\Xi$ and expectation bound $B$, we have an expectation constraint formula $\Psi=\forall \Xi\supset B$. The quantifier rank of an expectation constraint $\Psi$ is likewise equal to the number of distinct variables occurring in $\Xi$ and $B$ together. 
Again, the semantics is that for any substitution of names for variables $\theta$ satisfying $\Xi$, the substitution into the expectation bound $B\theta$ holds.

\noindent{\textbf{Knowledge Base and Query:}} In general, we define a knowledge base $\Delta$ as a finite non-empty set of
logical and expectation constraint formulas about various relations. 
A query $q$ is another constraint formula pending verification. 

Thus, assuming the knowledge base is satisfiable, checking whether the query is compatible with it is equivalent to checking the consistency of the expanded knowledge base $\Delta\cup q$. From now on, we will refer to 
$\Delta\cup q$ simply as the knowledge base for brevity.

\noindent{\textbf{Semantics:}}  A model $\mathcal{M}$ for the knowledge base $\Delta$ is given by a probability measure space, together with an assignment of a measurable function for each binding of a relational variable to a tuple of names of the appropriate arity, such that the logical constraints in $\Delta$ are satisfied with probability $1$, and the expectation bounds are also satisfied for all bindings.

Now we are ready to resolve queries using knowledge bases in the described form. For instance, as the expectation of a Boolean variable coincides with the probability of the variable taking the value one, we can directly reason about various conditional probabilities, as illustrated in the following example.

\begin{example} \label{em:LT} 


In the following example, we would like to speculate the likelihood of a war between Antony and Octavian. We will consider a binary Boolean relation \textbf{War} and an ternary Boolean relation \textbf{LoveTriangle}, with constants Antony and Cleopatra. We can encode an assertion like ``$\Pr[\text{War}(x,y)|\text{LoveTriangle}(x,y,z)]\geq .75$'' by using the definition of conditional probability to rewrite it as ``$\Pr[\text{War}(x,y)\land \text{LoveTriangle}(x,y,z)]\geq .75 \Pr[\text{LoveTriangle}(x,y,z)]$.'' Thus we obtain a knowledge base
\begin{equation*}
\small
\begin{aligned}
\forall \text{LoveTriangle}(x,y,z)^2-\text{LoveTriangle}(x,y,z) &= 0\\
\forall \text{War}(x,y)^2-\text{War}(x,y) &= 0\\
\forall e(\text{War}(x,y)\text{LoveTriangle}(x,y,z))\qquad&\\
-.75 e(\text{LoveTriangle}(x,y,z)) &\geq 0.
\end{aligned}
\end{equation*}
Given also: 
\begin{equation*}
\small
\begin{aligned}
    \forall x\neq Antony\land x\neq Cleopatra\\
    \supset e(\text{LoveTriangle}(x,Antony,&Cleopatra)\geq 1
\end{aligned}
\end{equation*}
this would imply that Antony has a greater than 75\% chance of being at war with everyone in the universe (apart from himself and Cleopatra), including Octavian.
\end{example}


\subsection{Propositional Sum-of-squares Refutations}
At the propositional level, a knowledge base consists solely of a collection of \textit{ground} logical constraints and \textit{ground} expectation bounds.
A sum-of-squares refutation demonstrates that it is impossible for a joint distribution on random variables to be consistent with the knowledge base, reducing reasoning to a moment problem. 

A convenient form for our purposes is due to Putinar~\shortcite{putinar}: suppose our knowledge base consists of logical inequalities $\{g_i\geq 0\}_{i\in I}$, equalities $\{h_j=0\}_{j\in J}$, and expectation bounds $\{b_k\geq 0\}_{k\in K}$, where the various $g_i$, $h_j$, and $b_k$ are written as polynomials in the relations. A sum-of-squares polynomial $\sigma$ is, as the name suggests, equal to a sum of squares of arbitrary polynomials $\sum_{\ell\in L}(p_\ell)^2$. Now, for sum-of-squares polynomials $\sigma_0$ and $\sigma_i$ for $i\in I$, arbitrary polynomials $q_j$ for $j\in J$, and nonnegative real numbers $r_k$ for $k\in K$, suppose we have an expression
\begin{equation}
\small
    \sigma_0+\sum_{i\in I}\sigma_ig_i+\sum_{j\in J}q_jh_j+\sum_{k\in K}r_kb_k
=-1 \label{e:sos_proof} \tag{1}.
\end{equation}

 This is called a sum-of-squares refutation. It is sound, since for any joint distribution on the relational variables the sum-of-squares polynomials must be nonnegative; assuming that the logical constraints hold, and hence that each $g_i$ is nonnegative and each $h_j$ is identically $0$ over the support of the distribution, we see that by linearity of the expectation operator, the sum-of-squares expression must be nonnegative. Since, it is formally equal to $-1$, which certainly has negative expectation.

In order to establish a sum-of-squares refutation, we need to find 
$\sigma_0$ and $\sigma_i$ for $i\in I$, arbitrary polynomials $q_j$ for $j\in J$, and nonnegative real numbers $r_k$ for $k\in K$ satisfying the above equality~\ref{e:sos_proof}. A manually established sum-of-squares refutation is provided below.

\begin{example}\label{ex:cheby}
We demonstrate a proof of Chebyshev's inequality in sum-of-squares. This is a standard tool in probability with a simple, low-degree proof. We suppose $X$ is a mean-$0$ random variable with variance $\lambda>0$. We define a tail event $T$ to be a Boolean (indicator) random variable corresponding to $|X|$ exceeding $\sqrt{k\lambda}$ -- or equivalently, $X^2\geq k\lambda$. For any given $\delta>0$, we'll prove that the probability of $T$ is at most $\frac{1}{k}+\delta$. Formally, now, we have the system:
\begin{equation*}
    \small
    \begin{aligned}
T(X^2-k\lambda) &\geq 0\\
(1-T)(k\lambda-X^2) &\geq 0\\
T^2-T &=0\\
e(X^2) - \lambda &= 0\\
e(X) &= 0\\
e(T) - (\frac{1}{k}+\delta) &\geq 0
\end{aligned}
    \end{equation*}
which we wish to refute. The sum-of-squares expression
\begin{equation*}
    \small
    \begin{aligned}
    & ((1-T)X)^2 + T(X^2 - k\lambda) - X^2(T^2 - T) \\
    & - (X^2 - \lambda) + k\lambda \left(T - \left( \frac{1}{k} + \delta \right) \right)
    \end{aligned}
    \end{equation*}
is formally equal to $-k\lambda\delta$, so by rescaling by $\frac{1}{k\lambda\delta}$, we obtain a sum-of-squares refutation. Observe that the total degree of the expression is 4.
\end{example}

\subsection{Tractability of Constant-Degree Sum-of-Squares (SOS)}

Deriving a sum-of-squares refutation is a non-trivial task, as demonstrated by Example~\ref{ex:cheby}. Luckily, as noted independently by many authors \cite{shor87,nesterov00,parrilo00,lasserre01}, for any fixed $d\in\mathbb{N}$, the fragment of ground sum-of-squares consisting of expressions of degree at most $d$ is tractable: suppose $\vec{R}$ is the vector of all ground monomials of relational random variables up to degree $d/2$ (assume $d$ is even), ordered by increasing total degree, including the degree-0 constant $1$. We refer to the expectation of the outer-product $\vec{R}\vec{R}^\top$ as a moment matrix $[M]$. For a polynomial $p$, we define the degree-$d$ localizing matrix $[pM]$ to be, for the largest index $s^*$ such that the degree of $R_{s^*}$ is at most $d/2-\deg(p)$, the expectation of $p\cdot \vec{R}_{1:s^*}\vec{R}_{1:s^*}^\top$. Observe that this matrix is constructed so that all expressions appearing in it have total degree at most $d$. Finally, the degree-$d$ sum-of-squares semidefinite program corresponding to logical inequalities $\{g_i\geq 0\}_{i\in I}$, equalities $\{h_j=0\}_{j\in J}$, and expectation bounds $\{b_k\geq 0\}_{k\in K}$ is
\begin{equation*}
\small
\begin{aligned}
[M]&\succeq 0\\
[g_iM]&\succeq 0&&i\in I\\
[h_jM]& = 0&&j\in J\\
b_k&\geq 0&&k\in K
\end{aligned}
\end{equation*}
Assuming that the logical inequalities assert that each term only obtains bounded values -- specifically, for all ground terms $\tau$, $\tau^2\leq U$ for some common $U$ -- this program is infeasible iff there is a degree-$d$ sum-of-squares refutation. (Such systems are said to be explicitly compact.) Indeed, degree-$d$ refutations are described by the dual semidefinite program. Hence, we can test for the existence of refutations with coefficients up to a given size using a polynomial-time algorithm for deciding the feasibility of this semidefinite program:

\begin{theorem}\textup{(Soundness \cite{shor87,nesterov00,parrilo00,lasserre01}).}\label{thm:sossound}
Let $\{g_i\geq 0\}_{i\in I}, \{h_j=0\}_{j\in J}, \{b_k\geq 0\}_{k\in K}$ be a system of constraints that is explicitly compact. Then either there is a degree-d sum-of-squares refutation or there is a solution to the degree-d sum-of-squares semidefinite program.
\end{theorem}

Lasserre~\shortcite{lasserre01} showed that ground sum-of-squares is complete for explicitly compact knowledge bases: for every inconsistent system, there is a sum-of-squares refutation that, in particular, has some finite degree. 

\begin{theorem}\textup{(Completeness, a corollary of \cite{putinar}).}\label{thm:soscomplete}
There exists a probability distribution with expected values $\{e(x^{\vec{\alpha}})\}_{\vec{\alpha} \in \mathbb{N}^n}$ supported on a set given by an explicitly compact system $\{g_i\geq 0\}_{i\in I}, \{h_j=0\}_{j\in J}, \{b_k\geq 0\}_{k\in K}$ iff every moment matrix is positive semidefinite, every localizing matrix for each $g_i$ is positive semidefinite, every localizing matrix for each $h_j$ is zero, and the bounds $b_k$ are nonnegative (i.e., solutions exist for all degrees d).
\end{theorem}

Naturally, there is no reasonable bound on the degree, and even for systems in which all of the relational variables are Boolean, the degree may need to be linear in the number of such variables; observe that in general, the number of monomials of a given degree is exponential in the degree. Hence, this corresponds to an exponential-size semidefinite program formulation.

But, the constant-degree fragments are already quite expressive. Many theorems of probabilistic analysis are captured by these fragments \cite{barak2012hypercontractivity,oz13,dmn13,kotz14}. More generally, Berkholz~\shortcite{berkholz18} showed that constant-degree fragments of ground sum-of-squares simulate the constant degree fragments of polynomial calculus~\cite{cei96}, and Juba~\shortcite{juba2019polynomial} showed that they also simulate bounded-space treelike resolution~\cite{et01}, which naturally generalizes unit propagation and other tractable fragments of resolution \cite{ablm08}.

We remark briefly that extensions of Theorem~\ref{thm:soscomplete} to countably-infinite dimension have been obtained; see \cite{ghasemi2016moment} and \cite{curto2023truncated}. While conceptually pleasing, as with the finite-dimensional (ground) case, these cannot provide practically usable bounds on the size of the refutation. For the sake of efficient algorithms, we will have to focus on a limited fragment.

\section{Efficient Reasoning in Open Universes}
From this point onward, we enter the first-order world with an \textit{open-universe} (OU) setting, where we do not assume a complete knowledge of all the objects in the universe. Thus, the ability to reason in an open universe demonstrates robustness and adaptability in the face of uncertainty. This assumption determines how we need to ground our knowledge base.

\subsection{Grounding in Probability Logic}
A ground theory is obtained from $\Delta$ by substituting variables with names. Let the rank of $\Delta$ be the maximum quantifier rank of any formula in $\Delta$. Given a knowledge base of this form, we are going to define a grounding of $\Delta$ as a substitution of all variables in $\Delta$ with names. 
\begin{definition}
\begin{equation*}
    \text{GND}(\Delta) = \{\phi\theta | [\forall \Xi\supset\phi] \in \Delta, \vDash \Xi\theta\}
\end{equation*} Typically, to limit the domain of variables used in a grounding we will use:
\begin{equation*}
    \text{GND}(\Delta,k) = \{\phi\theta |  [\forall \Xi\supset\phi] \in \Delta, \vDash \Xi\theta, \theta \in K\}
\end{equation*}where $K$ consists of all constants present in $\Delta$ plus $k$ additional \textit{generic} names, for $k \geq 0$. 
\end{definition}
The generic names capture the values the formula may take when its variables are bound to elements outside the set of constants. 
As opposed to the \textit{domain closure assumption}, where we do not consider any objects outside of the knowledge base (i.e., $DC(\Delta)=GND(\Delta,0)$), we will be capturing an open universe via $OU(\Delta)=GND(\Delta,k)$, where $k$ is exactly the rank of $\Delta$. Intuitively, we need at least $k$ distinct names since $\Xi$ can require that all variables are bound to distinct elements, but moreover we will see (cf.\ Theorem~\ref{thm:inf-pseudo}) that $k$ generic names are also sufficient to capture the behavior of all possible groundings.

\begin{example} \label{em:GroundLT}Continuing our previous example, we can get a clearer picture of Antony's ongoing wars. For the knowledge base $\Delta$ as defined in Example \ref{em:LT}, we get GND($\Delta$, 3) with 3 generic names $Octavian, Caesar, Cicero$ (ignoring the Boolean axioms and equality constraints):
\begin{equation*}
\small
\begin{aligned}
e(\text{War}(Oct., Ant.) \text{L.T.}(Oct., Ant., Cleo.))\qquad&\\
-.75 e(\text{L.T.}(Oct., Ant., Cleo.)) &\geq 0\\
e(\text{War}(Cae., Ant.) \text{L.T.}(Cae., Ant., Cleo.))\qquad&\\
-.75 e(\text{L.T.}(Cae., Ant., Cleo.)) &\geq 0\\
e(\text{War}(Cic., Ant.) \text{L.T.}(Cic., Ant., Cleo.))\qquad&\\
-.75 e(\text{L.T.}(Cic., Ant., Cleo.)) &\geq 0\\
e(\text{L.T.}(Oct., Ant., Cleo.)) -1 &\geq 0\\
e(\text{L.T.}(Cae., Ant., Cleo.)) -1 &\geq 0\\
e(\text{L.T.}(Cic., Ant., Cleo.)) -1 &\geq 0
\end{aligned}
\end{equation*}
\noindent
Observe, the more generic names we add, the more wars Antony likely ends up in. Imagine how bad his issues would be in GND($\Delta$) 
\textup{(The full grounding including the Boolean axioms is included in the Appendix~\ref{appendix:full} for completeness)}.
\end{example}

\subsection{Satisfiability}
In this section we cast the usual notion of logical satisfiability into the algebraic language of sum-of-squares feasibility.  Concretely, we say that a knowledge base~$\Delta$ is \emph{satisfiable} if and only if its associated sum-of-squares semidefinite program admits a feasible solution, which we call a pseudomodel:
\begin{definition}
A degree-$d$ \textit{pseudomodel} for $\Delta$ is given by an assignment of a real number to each $e(\mu)$ for each monomial $\mu$ up to degree $d$, such that the assignments satisfy the infinite sum-of-squares program.
\end{definition}
  For a finite set of names, used as the domain of quantification for GND($\Delta$), the corresponding program is a (finite) ground sum-of-squares program Theorem~\ref{thm:soscomplete} asserts that for explicitly compact $\Delta$ and sufficiently high degree, satisfiability w.r.t.\ models and pseudomodels coincide.
  This equivalence lets us reduce logical satisfiability questions to the feasibility of polynomial-size semidefinite programs, paving the way for tractable relational inference in open universes. 

\subsection{Equivalence Classes}
\begin{definition}
    Two ground monomials $\mu$ and $\mu'$ are said to be in the same renaming equivalence class if there is a renaming substitution $\theta$ such that $\mu\theta=\mu'$.
\end{definition}

\begin{example}
    For $\Delta$ containing: 
    \begin{equation*}
    \small
    \begin{aligned}
        \forall e(Q(x,y)) - 3 &\geq 0 \\
        \forall e(Q(x, james)) &\geq 0 \\
        e(P(james)) - 1 &\geq 0
    \end{aligned}
    \end{equation*}
    and the subsequent grounding GND($\Delta$,2): 
    \begin{equation*}
    \small
    \begin{aligned}
         e(Q(jack, jill)) - 3\geq 0\\
         e(Q(jill, jack)) - 3 \geq 0\\
         e(Q(james, jack)) - 3 \geq 0\\
         e(Q(james, jill)) - 3 \geq 0\\
         e(Q(jack, james)) \geq 0\\
         e(Q(jill, james)) \geq 0\\
         e(P(james)) - 1\geq 0\\
    \end{aligned}
    \end{equation*}

    We get that $Q(jack, jill)$ and $Q(jill, jack)$ are in the same equivalence class, but $Q(james, jack)$ is not because james is not a generic name, and thus can't be freely renamed.
    
\end{example}

Each name will fall into a single equivalence class, and it will be used when lifting our grounded logic.

\subsection{Lifted Sum-of-Squares}

We now present our lifted sum-of-squares system.
Let $\Delta$ be a knowledge base consisting of logical constraint formulas and expectation bounds, and let $k$ be the quantifier rank of $\Delta$. 
We obtain our \emph{lifted} sum-of-squares system by adding equality constraints for the generic names to the knowledge base.

\begin{definition}[Lifted Sum-of-Squares]
  For any given degree bound $d$, \emph{degree-$d$ lifted sum-of-squares} for a first-order knowledge base $\Delta$ uses the language of propositional degree-$d$ sum-of-squares with the following propositional knowledge base:  $\mathrm {GND}(\Delta,k)$ union with the set of equality constraints $e(\mu)-e(\mu')=0$ for each pair of ground monomials $\mu,\mu'$ in the names used by $\mathrm {GND}(\Delta,k)$ of degree up to $d$ such that for a renaming substitution $\theta$, $\mu=\mu'\theta$. We denote this propositional sum-of-squares knowledge base by $\mathrm {GND}_{\mathrm {lifted SOS}}(\Delta,k)$. Thus, a degree-$d$ lifted sum-of-squares refutation of $\Delta$ means a propositional degree-$d$ sum-of-squares refutation of $\mathrm {GND}_{\mathrm {lifted SOS}}(\Delta,k)$.
\end{definition}
\noindent
Observe that since degree-$d$ refutations of $\Delta$ are simply degree-$d$ ground sum-of-squares refutations, it follows immediately from Theorem \ref{thm:sossound} that refutations exist iff the sum-of-squares semidefinite program is infeasible. This program na\"{\i}vely has dimension equal to the number of monomials that can be constructed from relational variables bound to names from $\mathcal{C}$ or our $k$ generic names, which is polynomial in $|\mathcal{C}|$ and the number of relation symbols as long as both $k$ and the arity of all relations is bounded by a constant. But, instead of including the equality constraints $e(\mu)-e(\mu')=0$ for each pair of equivalent ground monomials $\mu,\mu'$, we can use a single variable to represent the value for the entire equivalence class. Then the number of constraints is similarly bounded, yielding a polynomial time guarantee for our lifted inference:



\begin{theorem}
    Given a $\Delta$ with $c$ constants, $n$ $\forall$-clauses, each mentioning at most $m$ predicates, and with rank $k$, the running time for a lifted degree-$d$ sum-of-squares inference in a open universe is polynomial in $c,n$, and $m$ for fixed $k$ and $d.$
\end{theorem}
\begin{proof}
    For each $\forall$-clause we will have at most $(c + k)^k$ substitutions for each predicate. Then the total number of atoms in GND($\Delta$, $k$) is $O(nm(c + k)^k)$.
    
    As there are algorithms solving a semidefinite program up to arbitrary precision in polynomial time with respect to the size of the program (e.g. the ellipsoid method \cite{bland1981ellipsoid}),  we thus can decide whether or not a degree-$d$ refutation exists in polynomial time in $c,n$, and $m$ for fixed $k$ and $d$.
\end{proof}

For example, we can carry out the following probabilistic inference in polynomial time.

\begin{example}\label{em:sos}
 We're going to assume a grounding that includes the 3 generic names Oct., Cic., and Cae., (As partially shown in Example~\ref{em:GroundLT}) and infer a bound on the value of e(War(Ant.,Oct.)). Given our boolean constraints on both relations War and L.T.. We can then infer that the expression $e(War(Ant., Oct.)) \ge.75$ via the sum-of-squares expression (denoting e(War(Ant.,Oct.)) by $W$ and e(LoveTriangle(Ant.,Oct.,Cleo.)) by $LT$), 
 
 \begin{equation*}
 \small
 \begin{aligned}
     W^2(1-LT)^2 + W^2(LT-LT^2) + (1-LT)(W-W^2) \\
     + (W\cdot LT-.75LT) + (.75LT - .75)= W - .75
 \end{aligned}
 \end{equation*} \textup{(See Appendix~\ref{appendix:steps} for details.)}
\end{example}

\subsection{Soundness and Completeness}
We will now show that our lifted sum-of-squares logic is both sound and complete for a GND($\Delta$) that uses an open-universe. To begin, recall the results on ground sum-of-squares, Theorems~\ref{thm:sossound} and \ref{thm:soscomplete}.
%
%
We will leverage these two theorems to show that our lifted sum-of-squares logic is sound and complete when finding a satisfying model of the equivalence classes, assuming explicit compactness. 

We first argue that pseudomodels can be taken to agree within an equivalence class without loss of generality.

\begin{lemma}\label{lem:convex}
  If GND$(\Delta,k)$ is satisfiable then there exists a pseudomodel $\mathcal{M}$ in which all members of each equivalence class take the same value.
\end{lemma}
\begin{proof}
Assume GND$(\Delta,k)$ is satisfiable.  Then for its associated semidefinite program there is some feasible “generic” assignment $e_{\mathcal{G}}$ of moment‐variables. Now, consider any renaming substitution $\theta$ that takes the $k$ generic names in GND($\Delta,k$) to the same set of $k$ generic names. For any constraint $\forall \Xi\supset\phi$ in $\Delta$, for any grounding $\theta'$ such that $\vDash \Xi\theta'$, observe that $\vDash \Xi\theta'\theta$ as well: indeed, for any atom of $\Xi$, if $\theta'$ grounds a variable to a constant, $\theta$ does not rename the constant, so equality w.r.t.\ any constant is unchanged; and $\theta'$ grounds both variables to the same name iff $\theta$ maps both to the same name, so equality between generic names is also preserved. Therefore, $\phi\theta'\theta$ is also in GND($\Delta,k$). Now consider the assignment $\mathcal{M}$ that assigns each $e_{\mathcal{M}}(\mu)$ the value $\frac{1}{k!}\sum_{\text{renaming }\theta}e_{\mathcal{G}}(\mu\theta)$. This $e_{\mathcal{M}}(\mu)$ is then going to be our proposed pseudomodel. Observe that for any renaming $\theta''$, 
\begin{equation*}
\small
e_{\mathcal{M}}(\mu\theta'')=\frac{1}{k!}\sum_{\theta}e_{\mathcal{G}}(\mu\theta''\theta)
=\frac{1}{k!}\sum_{\theta}e_{\mathcal{G}}(\mu\theta)=e_{\mathcal{M}}(\mu)
\end{equation*}
so indeed, the equivalence classes share a common value. Likewise, since $\varphi\theta''$ is in GND($\Delta,k$), the assignments $e_{\mathcal{G}\theta''}(\mu)=e_{\mathcal{G}}(\mu\theta'')$ are also solutions. Moreover, since $e_{\mathcal{M}}$ is a convex combination of these $e_{\mathcal{G}\theta''}$, and the feasible region of any semidefinite program is convex, $e_{\mathcal{M}}$ must also be feasible. Therefore, $e_{\mathcal{M}}$ is indeed a pseudomodel.
\end{proof}

Next, we observe that a set of generic names equal to the quantifier rank suffices to obtain representatives of all of the equivalence classes of monomials that appear in GND($\Delta$); we will use this property to extend to a pseudomodel for the entire GND($\Delta$) next.

\begin{lemma}\label{lem:classreps}
    For a given $\Delta$ a pseudomodel for GND($\Delta$, k), where k is the quantifier rank of $\Delta$, assigns a value for some grounding from every equivalence class occurring in GND($\Delta$).
\end{lemma}

\begin{proof}
Consider any grounding $\theta$ such that for $\forall\Xi\supset\phi$ in $\Delta$, $\phi\theta$ is in GND($\Delta$). Since $\forall\Xi\supset\phi$ has quantifier rank $k$, at most $k$ names from $\mathcal{G}$ are assigned by $\theta$ in $\phi\theta$. So, consider the renaming substitution $\theta''$ that takes these $k$ names to the $k$ names used in GND($\Delta$, k) (arbitrarily). Observe that $\phi\theta\theta''$ is in GND($\Delta$, k), so the pseudomodel for GND($\Delta$, k) assigns a value to each monomial in $\phi\theta\theta''$.  For each monomial occurring in $\phi\theta$, there is an equivalent monomial in $\phi\theta\theta''$ that has therefore indeed been assigned a value in the pseudomodel.
\end{proof}

\begin{theorem}\label{thm:inf-pseudo}
For all $d$, GND($\Delta$) has a degree-$d$ pseudomodel iff GND($\Delta$, $k$) has a degree-$d$ pseudomodel, where $k$ is equal to the rank of $\Delta$.
\end{theorem}

\begin{proof}
    The direction of most interest is showing that given GND($\Delta$, $k$) has a pseudomodel then GND($\Delta$) also has a pseudomodel. Starting with a pseudomodel $\mathcal{M}$ that satisfies GND($\Delta$, $k$) we can apply Lemma \ref{lem:convex} to get $\mathcal{M}'$ which will be our pseudomodel where all members of an equivalence class take the same value. 
   By Lemma \ref{lem:classreps} we have that each monomial in GND($\Delta$) belongs to an equivalence class with a representative appearing in GND($\Delta$,k). This guarantees a pseudomodel exists for GND($\Delta$), because for every $\phi\theta$ in GND($\Delta$), we assign the same values to its monomials as in $\phi\theta\theta''$ where $\theta''$ is the renaming substitution that assigns the $k$ names in $\phi\theta$ to the names in  GND($\Delta$, $k$). Since the constraints of the semidefinite program associated with $\phi\theta\theta''$ are satisfied by these values, so are the constraints associated with $\phi\theta$.
\end{proof}


\begin{theorem}
    Given an explicitly compact knowledge base $\Delta$ and a grounding GND($\Delta$, $k$), a lifted sum-of-squares program using GND($\Delta$, $k$) is sound and complete for degree-$d$ refutations of $\Delta$.
\end{theorem}

\begin{proof}
    For soundness, we first observe that if $\Delta$ is satisfiable, then the degree-$d$ lifted sum-of-squares program for GND($\Delta$, $k$) is feasible by Lemma~\ref{lem:convex}, and hence by Theorem~\ref{thm:sossound} there is no degree-$d$ lifted sum-of-squares refutation of $\Delta$.

    For completeness with respect to degree-$d$, suppose that there is no degree-$d$ pseudomodel of GND$(\Delta)$. Then by Theorem~\ref{thm:inf-pseudo}, there is no degree-$d$ pseudomodel of GND($\Delta$, $k$), either. By Theorem~\ref{thm:sossound}, there is a degree-$d$ sum-of-squares refutation of GND($\Delta$, $k$); by renaming, this is a degree-$d$ lifted sum-of-squares refutation of $\Delta$.
%
\end{proof}

\subsection{Discussion of Tractability}

We stress that some probabilistic inference problems are inherently intractable and we do not claim to solve these problems in polynomial time. We know that even approximate inference in Bayesian belief networks is NP-hard \cite{dagum1993approximating}; but this was due to the fact that any language capable of expressing all 3-CNFs can encode NP-hard problems. Thus, restricting the language seems to be an unpromising approach to obtaining tractable inference. 

Indeed, recall that the same issue arises in even in classical Boolean inference. There, the approach broadly pursued in ``SAT solvers'' in practice does not attempt to restrict the language, but rather the algorithm is only feasible to run on a subset of the instances. Following \cite{beame2004towards}, we know that the subset of tractable instances is captured by a fragment of resolution in which the proof size is small.

Here, we pursue an analogous approach to probabilistic inference. We have opted for a language that can also represent all 3-CNFs and so in general (in the absence of a degree bound) we encounter an intractable inference problem. As with Boolean reasoning, we can guarantee that our algorithm is tractable for clearly defined fragments (of sum-of-squares), but we do not know in advance what moment degrees will suffice to solve the problem. Unfortunately, even for Boolean inference, SAT solvers face the analogous difficulty.

\section{Comparison to Other Models}

We will use three existing models for probabilistic inference and show where lifted sum-of-squares logic improves in either tractability, expressiveness, or both. A short summary of these comparisons can be seen in Table 1.

\begin{table}[H]
    \centering
    \begin{tabular}{|c|c|c|}
    \hline
        \textit{Model} & Ternary Relations & Full Arithmetic\\
        \hline
        Lifted SOS & \ding{51}& \ding{51}\\
        WFOMC-OU &\ding{55}& \ding{55}\\
        PSL & \ding{51} & \ding{55}\\
        TML & \ding{51}& \ding{51}\\
        \hline
    \end{tabular}
    \caption{Comparisons of tractabilities among existing statistical relational learning methods}
    \label{tab:my_label}
\end{table}

\subsection{Tractable Markov Logic (TML)}
Tractable Markov Logic (TML), as defined in \cite{domingos2012tractable} has two very distinct characteristics. First of all, they require the objects all be inside a single hierarchy, which is highly restrictive. Secondly, they adopt weights together with those sub-class structures to maintain tractability. The explicitly strong structural assumption is not obvious for many cases including our example. Along with these drawbacks there is a fundamental difference between TML and Lifted Sum-of-Squares. Markov Logics in general focus only on bounding the maximum entropy distribution, whereas we focus on bounding all distributions. This can be particularly useful when using these models to reject a null hypothesis. By having certifiable bounds, our model can reject a hypothesis by itself without further experimentation being necessary. An example of this would be the biological benchmarks used in \cite{ribeiro2022learning}.

\subsection{Weighted First-Order Model Counting with an Open Universe (WFOMC-OU)}

As mentioned above, Belle's approach of performing weighted first-order model counting in an open universe \cite{belle2017open} is believed to be intractable for many ternary relations and beyond. This means that Example~\ref{em:sos} is likely to be intractable using this approach. However, our runtime guarantees are polynomial in the rank of our knowledge base meaning this example is well within our tractable fragment. Along with this, many model counting algorithms rely on approximation to increase their tractability. While this may be useful for some domains it runs into a similar issue as TML. If an exact answer is necessary these models may be unable to provide an answer.

\subsection{Probabilisitc Soft Logic (PSL)}

This model, introduced under the name Probabilistic Soft Logic, is another commonly applied method to solving probabilistic reasoning in relational domains \cite{kimmig2012short}. Much similar to Markov Logic Networks, Probabilisitc Soft Logic combines graphical models and first-order logic, but allows the truth values to be soft as in any time between zero and one. These methods have had success in practice but currently lack the theoretical guarantees provided in our work. In addition to this, PSL is limited to arithmetic rules of linear combinations of relations. This is in contrast to our Example 2, where we are able to leverage the term $War(x,y)L.T.(x,y,z)$ to create a rule that models conditional probability. PSL models are also an example of a model that utilizes weighted counting to handle the probabilitistic nature. This leads to many of the same drawbacks that were mentioned for WFOMC-OU.

\section{Future Work and Conclusion}
In summary our work is able to extend previous work on probabilistic inference using sum-of-squares to a first order logic. This allows for a substantially larger tractable fragment than previous work. While approach manages to employ first-order reasoning, but there may still be more simplification available for the matrices resulting from our semi-definite program. By utilizing an approach similar to that of~\cite{kmt17} for relational linear programming, one may be able to simplify the semi-definite program even more. Further, if sparsity exists in the formulation, we may also exploit it using methods mentioned in~\cite{lasserre06}. One downside of our approach is that it inherits the limitation that sum-of-squares cannot represent or reason about independence of random variables. Although the most straightforward approach to representing independence leads immediately to an intractable polynomial optimization problem, it is natural to ask if this is truly inherent.

\section*{Acknowledgments}
This work is supported by the National Science Foundation under Grant IIS-2214141 and IIS-1942336.

\bibliographystyle{named}
\bibliography{references.bib}

\begin{thebibliography}{}

\bibitem[\protect\citeauthoryear{Ans\'{o}tegui \bgroup \em et al.\egroup }{2008}]{ablm08}
Carlos Ans\'{o}tegui, Mar{\'{\i}}a~Luisa Bonet, Jordi Levy, and Felip Many\'{a}.
\newblock Measuring the hardness of {SAT} instances.
\newblock In {\em Proc. AAAI'08}, pages 222--228, 2008.

\bibitem[\protect\citeauthoryear{Barak \bgroup \em et al.\egroup }{2012}]{barak2012hypercontractivity}
Boaz Barak, Fernando~GSL Brandao, Aram~W Harrow, Jonathan Kelner, David Steurer, and Yuan Zhou.
\newblock Hypercontractivity, sum-of-squares proofs, and their applications.
\newblock In {\em Proceedings of the forty-fourth annual ACM symposium on Theory of computing}, pages 307--326, 2012.

\bibitem[\protect\citeauthoryear{Beame \bgroup \em et al.\egroup }{2004}]{beame2004towards}
Paul Beame, Henry Kautz, and Ashish Sabharwal.
\newblock Towards understanding and harnessing the potential of clause learning.
\newblock {\em Journal of artificial intelligence research}, 22:319--351, 2004.

\bibitem[\protect\citeauthoryear{Beame \bgroup \em et al.\egroup }{2015}]{beame2015symmetric}
Paul Beame, Guy Van~den Broeck, Eric Gribkoff, and Dan Suciu.
\newblock Symmetric weighted first-order model counting.
\newblock In {\em Proceedings of the 34th ACM SIGMOD-SIGACT-SIGAI Symposium on Principles of Database Systems}, pages 313--328, 2015.

\bibitem[\protect\citeauthoryear{Belle \bgroup \em et al.\egroup }{2015}]{belle2015probabilistic}
Vaishak Belle, Andrea Passerini, and Guy Van~den Broeck.
\newblock Probabilistic inference in hybrid domains by weighted model integration.
\newblock In {\em Proceedings of 24th International Joint Conference on Artificial Intelligence (IJCAI)}, pages 2770--2776. AAAI Press/International Joint Conferences on Artificial Intelligence, 2015.

\bibitem[\protect\citeauthoryear{Belle}{2017}]{belle2017open}
Vaishak Belle.
\newblock Open-universe weighted model counting.
\newblock In {\em Proceedings of the AAAI Conference on Artificial Intelligence}, volume~31, 2017.

\bibitem[\protect\citeauthoryear{Berkholz}{2018}]{berkholz18}
Christoph Berkholz.
\newblock The relation between polynomial calculus, {S}herali-{A}dams, and sum-of-squares proofs.
\newblock In {\em Proc. 35th STACS}, LIPIcs, pages 11:1--11:14, 2018.

\bibitem[\protect\citeauthoryear{Bland \bgroup \em et al.\egroup }{1981}]{bland1981ellipsoid}
Robert~G Bland, Donald Goldfarb, and Michael~J Todd.
\newblock The ellipsoid method: A survey.
\newblock {\em Operations research}, 29(6):1039--1091, 1981.

\bibitem[\protect\citeauthoryear{Carbonetto \bgroup \em et al.\egroup }{2005}]{carbonetto2005nonparametric}
Peter Carbonetto, Jacek Kisy{\'n}ski, Nando de~Freitas, and David Poole.
\newblock Nonparametric bayesian logic.
\newblock In {\em Proceedings of the Twenty-First Conference on Uncertainty in Artificial Intelligence}, pages 85--93, 2005.

\bibitem[\protect\citeauthoryear{Chavira and Darwiche}{2008}]{cd08}
Mark Chavira and Adnan Darwiche.
\newblock On probabilistic inference by weighted model counting.
\newblock {\em Artificial Intelligence}, 172(6--7):772--799, 2008.

\bibitem[\protect\citeauthoryear{Clegg \bgroup \em et al.\egroup }{1996}]{cei96}
Matthew Clegg, Jeff Edmonds, and Russell Impagliazzo.
\newblock Using the {G}r\"{o}bner basis algorithm to find proofs of unsatisfiability.
\newblock In {\em Proc. 28th STOC}, pages 174--183, 1996.

\bibitem[\protect\citeauthoryear{Curto \bgroup \em et al.\egroup }{2023}]{curto2023truncated}
Raul~E Curto, Mehdi Ghasemi, Maria Infusino, and Salma Kuhlmann.
\newblock The truncated moment problem for unital commutative $\mathbb{R}$-algebras.
\newblock {\em Journal of Operator Theory}, 90(1):223--261, 2023.

\bibitem[\protect\citeauthoryear{Dagum and Luby}{1993}]{dagum1993approximating}
Paul Dagum and Michael Luby.
\newblock Approximating probabilistic inference in bayesian belief networks is np-hard.
\newblock {\em Artificial intelligence}, 60(1):141--153, 1993.

\bibitem[\protect\citeauthoryear{De \bgroup \em et al.\egroup }{2013}]{dmn13}
Anindya De, Elchanan Mossel, and Joe Neeman.
\newblock Majority is stablest: discrete and {SoS}.
\newblock In {\em Proc. 45th STOC}, pages 477--486, 2013.

\bibitem[\protect\citeauthoryear{Domingos and Webb}{2012}]{domingos2012tractable}
Pedro Domingos and William Webb.
\newblock A tractable first-order probabilistic logic.
\newblock In {\em Proceedings of the AAAI Conference on Artificial Intelligence}, volume~26, pages 1902--1909, 2012.

\bibitem[\protect\citeauthoryear{Esteban and Tor{\'{a}}n}{2001}]{et01}
Juan~Luis Esteban and Jacobo Tor{\'{a}}n.
\newblock Space bounds for resolution.
\newblock {\em Inf. Comp.}, 171(1):84--97, 2001.

\bibitem[\protect\citeauthoryear{Feldstein and Belle}{2021}]{feldstein2021lifted}
Jonathan Feldstein and Vaishak Belle.
\newblock Lifted reasoning meets weighted model integration.
\newblock In {\em Uncertainty in Artificial Intelligence}, pages 322--332. PMLR, 2021.

\bibitem[\protect\citeauthoryear{Ghasemi \bgroup \em et al.\egroup }{2016}]{ghasemi2016moment}
Mehdi Ghasemi, Salma Kuhlmann, and Murray Marshall.
\newblock Moment problem in infinitely many variables.
\newblock {\em Israel Journal of Mathematics}, 212:1012--1012, 2016.

\bibitem[\protect\citeauthoryear{Halpern and Pucella}{2007}]{halpern2007characterizing}
Joseph~Y Halpern and Riccardo Pucella.
\newblock Characterizing and reasoning about probabilistic and non-probabilistic expectation.
\newblock {\em Journal of the ACM (JACM)}, 54(3):15--es, 2007.

\bibitem[\protect\citeauthoryear{Halpern}{1990}]{halpern1990analysis}
Joseph~Y Halpern.
\newblock An analysis of first-order logics of probability.
\newblock {\em Artificial intelligence}, 46(3):311--350, 1990.

\bibitem[\protect\citeauthoryear{Hummel and Holyoak}{2003}]{hummel2003symbolic}
John~E Hummel and Keith~J Holyoak.
\newblock A symbolic-connectionist theory of relational inference and generalization.
\newblock {\em Psychological review}, 110(2):220, 2003.

\bibitem[\protect\citeauthoryear{Juba}{2019}]{juba2019polynomial}
Brendan Juba.
\newblock Polynomial-time probabilistic reasoning with partial observations via implicit learning in probability logics.
\newblock In {\em Proceedings of the AAAI Conference on Artificial Intelligence}, volume~33, pages 7866--7875, 2019.

\bibitem[\protect\citeauthoryear{Kauers \bgroup \em et al.\egroup }{2014}]{kotz14}
Manuel Kauers, Ryan O'Donnell, Li-Yang Tan, and Yuan Zhou.
\newblock Hypercontractive inequalities via {SOS}, and the {Frankl}-{R\"{o}dl} graph.
\newblock In {\em Proc. 25th SODA}, pages 1644--1658, 2014.

\bibitem[\protect\citeauthoryear{Kazemi \bgroup \em et al.\egroup }{2016}]{kazemi2016new}
Seyed~Mehran Kazemi, Angelika Kimmig, Guy Van~den Broeck, and David Poole.
\newblock New liftable classes for first-order probabilistic inference.
\newblock {\em Advances in Neural Information Processing Systems}, 29, 2016.

\bibitem[\protect\citeauthoryear{Kazemi \bgroup \em et al.\egroup }{2017}]{kazemi2017domain}
Seyed~Mehran Kazemi, Angelika Kimmig, Guy Van~den Broeck, and David Poole.
\newblock Domain recursion for lifted inference with existential quantifiers.
\newblock {\em arXiv preprint arXiv:1707.07763}, 2017.

\bibitem[\protect\citeauthoryear{Kersting \bgroup \em et al.\egroup }{2017}]{kmt17}
Kristian Kersting, Martin Mladenov, and Pavel Tokmakov.
\newblock Relational linear programming.
\newblock {\em Artificial Intelligence}, 244:188--216, 2017.

\bibitem[\protect\citeauthoryear{Kimmig \bgroup \em et al.\egroup }{2012}]{kimmig2012short}
Angelika Kimmig, Stephen Bach, Matthias Broecheler, Bert Huang, and Lise Getoor.
\newblock A short introduction to probabilistic soft logic.
\newblock In {\em NIPS Workshop on probabilistic programming: Foundations and applications}, volume~1, page~3, 2012.

\bibitem[\protect\citeauthoryear{Krawczyk}{2012}]{krawczyk2012cognition}
Daniel~C Krawczyk.
\newblock The cognition and neuroscience of relational reasoning.
\newblock {\em Brain research}, 1428:13--23, 2012.

\bibitem[\protect\citeauthoryear{Kuusisto and Lutz}{2018}]{kuusisto2018weighted}
Antti Kuusisto and Carsten Lutz.
\newblock Weighted model counting beyond two-variable logic.
\newblock In {\em Proceedings of the 33rd Annual ACM/IEEE Symposium on Logic in Computer Science}, pages 619--628, 2018.

\bibitem[\protect\citeauthoryear{Ku{\v{z}}elka \bgroup \em et al.\egroup }{2018}]{kuzelka2018relational}
Ond{\v{r}}ej Ku{\v{z}}elka, Yuyi Wang, Jesse Davis, and Steven Schockaert.
\newblock Relational marginal problems: Theory and estimation.
\newblock In {\em Proceedings of the AAAI Conference on Artificial Intelligence}, volume~32, 2018.

\bibitem[\protect\citeauthoryear{Kuzelka}{2021}]{kuzelka2021weighted}
Ondrej Kuzelka.
\newblock Weighted first-order model counting in the two-variable fragment with counting quantifiers.
\newblock {\em Journal of Artificial Intelligence Research}, 70:1281--1307, 2021.

\bibitem[\protect\citeauthoryear{Lakemeyer and Levesque}{2002}]{lakemeyer2002evaluation}
Gerhard Lakemeyer and Hector~J Levesque.
\newblock Evaluation-based reasoning with disjunctive information in first-order knowledge bases.
\newblock In {\em KR}, pages 73--81, 2002.

\bibitem[\protect\citeauthoryear{Lasserre}{2001}]{lasserre01}
Jean~Bernard Lasserre.
\newblock Global optimzation with polynomials and the problem of moments.
\newblock {\em SIAM J. Optimization}, 11(3):796--817, 2001.

\bibitem[\protect\citeauthoryear{Lasserre}{2006}]{lasserre06}
Jean~B Lasserre.
\newblock Convergent sdp-relaxations in polynomial optimization with sparsity.
\newblock {\em SIAM Journal on optimization}, 17(3):822--843, 2006.

\bibitem[\protect\citeauthoryear{Milch \bgroup \em et al.\egroup }{2005}]{milch2005blog}
Brian Milch, Bhaskara Marthi, Stuart Russell, David Sontag, Daniel~L Ong, and Andrey Kolobov.
\newblock Blog: Probabilistic models with unknown objects.
\newblock In {\em IJCAI International Joint Conference on Artificial Intelligence}, pages 1352--1359, 2005.

\bibitem[\protect\citeauthoryear{Mladenov \bgroup \em et al.\egroup }{2017}]{mladenov2017lifted}
Martin Mladenov, Leonard Kleinhans, and Kristian Kersting.
\newblock Lifted inference for convex quadratic programs.
\newblock In {\em Proceedings of the AAAI Conference on Artificial Intelligence}, volume~31, 2017.

\bibitem[\protect\citeauthoryear{Nesterov}{2000}]{nesterov00}
Y.~Nesterov.
\newblock Squared functional systems and optimization problems.
\newblock {\em High performance optimization}, 13:405--440, 2000.

\bibitem[\protect\citeauthoryear{O'Donnell and Zhou}{2013}]{oz13}
Ryan O'Donnell and Yuan Zhou.
\newblock Approximability and proof complexity.
\newblock In {\em Proc. 24th SODA}, pages 1537--1556, 2013.

\bibitem[\protect\citeauthoryear{Parrilo}{2000}]{parrilo00}
Pablo~A. Parrilo.
\newblock {\em Structured semidefinite programs and semialgebraic geometry methods in robustness and optimization}.
\newblock PhD thesis, California Institute of Technology, 2000.

\bibitem[\protect\citeauthoryear{Poole}{2003}]{poole2003first}
David Poole.
\newblock First-order probabilistic inference.
\newblock In {\em Proceedings of the 18th international joint conference on Artificial intelligence}, pages 985--991, 2003.

\bibitem[\protect\citeauthoryear{Poole}{2008}]{poole2008independent}
David Poole.
\newblock The independent choice logic and beyond.
\newblock In {\em Probabilistic inductive logic programming: theory and applications}, pages 222--243. Springer, 2008.

\bibitem[\protect\citeauthoryear{Putinar}{1993}]{putinar}
M.~Putinar.
\newblock Positive polynomials on compact semi-algebraic sets.
\newblock {\em Indiana U. Math. J.}, 42:969--984, 1993.

\bibitem[\protect\citeauthoryear{Ribeiro \bgroup \em et al.\egroup }{2022}]{ribeiro2022learning}
Tony Ribeiro, Maxime Folschette, Morgan Magnin, and Katsumi Inoue.
\newblock Learning any memory-less discrete semantics for dynamical systems represented by logic programs.
\newblock {\em Machine Learning}, pages 1--78, 2022.

\bibitem[\protect\citeauthoryear{Richardson and Domingos}{2006}]{richardson2006markov}
Matthew Richardson and Pedro Domingos.
\newblock Markov logic networks.
\newblock {\em Machine learning}, 62:107--136, 2006.

\bibitem[\protect\citeauthoryear{Shor}{1987}]{shor87}
N.~Shor.
\newblock An approach to obtaining global extremums in polynomial mathematical programming problems.
\newblock {\em Cybernetics and Systems Analysis}, 23(5):695--700, 1987.

\bibitem[\protect\citeauthoryear{T{\'o}th and Ku{\v{z}}elka}{2023}]{toth2023lifted}
Jan T{\'o}th and Ond{\v{r}}ej Ku{\v{z}}elka.
\newblock Lifted inference with linear order axiom.
\newblock In {\em Proceedings of the AAAI Conference on Artificial Intelligence}, volume~37, pages 12295--12304, 2023.

\bibitem[\protect\citeauthoryear{Van~Bremen and Ku{\v{z}}elka}{2023}]{van2023lifted}
Timothy Van~Bremen and Ond{\v{r}}ej Ku{\v{z}}elka.
\newblock Lifted inference with tree axioms.
\newblock {\em Artificial Intelligence}, 324:103997, 2023.

\bibitem[\protect\citeauthoryear{Van~den Broeck \bgroup \em et al.\egroup }{2011}]{van2011lifted}
Guy Van~den Broeck, Nima Taghipour, Wannes Meert, Jesse Davis, and Luc De~Raedt.
\newblock Lifted probabilistic inference by first-order knowledge compilation.
\newblock In {\em IJCAI}, pages 2178--2185, 2011.

\bibitem[\protect\citeauthoryear{Van~den Broeck \bgroup \em et al.\egroup }{2014}]{van2014skolemization}
Guy Van~den Broeck, Wannes Meert, and Adnan Darwiche.
\newblock Skolemization for weighted first-order model counting.
\newblock In {\em Fourteenth International Conference on the Principles of Knowledge Representation and Reasoning}, 2014.

\bibitem[\protect\citeauthoryear{Van~den Broeck \bgroup \em et al.\egroup }{2021}]{van2021introduction}
Guy Van~den Broeck, Kristian Kersting, Sriraam Natarajan, and David Poole.
\newblock {\em An Introduction to Lifted Probabilistic Inference}.
\newblock MIT Press, 2021.

\bibitem[\protect\citeauthoryear{Zeng \bgroup \em et al.\egroup }{2020}]{zeng2020probabilistic}
Zhe Zeng, Paolo Morettin, Fanqi Yan, Antonio Vergari, and Guy Van~den Broeck.
\newblock Probabilistic inference with algebraic constraints: Theoretical limits and practical approximations.
\newblock {\em Advances in Neural Information Processing Systems}, 33:11564--11575, 2020.

\end{thebibliography}

\clearpage
\newpage
\appendix

\section{Full Grounded Knowledge Base}
\label{appendix:full}

\begin{align*}
\small
War(Ant., Oct.)-War(Ant., Oct.)^2 &= 0\\
War(Ant., Cleo.)-War(Ant., Cleo.)^2 &= 0\\
War(Ant., Cic.)-War(Ant., Cic.)^2 &= 0\\
War(Ant., Cae.)-War(Ant.,Cae.)^2 &= 0\\
War(Cleo., Oct.)-War(Cleo., Oct.)^2 &= 0\\
War(Oct., Cic.)-War(Oct.,Cic.)^2 &= 0\\
War(Oct., Cae.)-War(Oct.,Cae.)^2 &=0\\
War(Cleo., Cic.)-War(Cleo., Cic.)^2 &=0\\
War(Cleo., Cae.)-War(Cleo., Cae.)^2 &=0\\
War(Cic., Cae.)-War(Cic., Cae.)^2 &= 0\\
L.T(Ant., Oct., Cleo)-L.T(Ant., Oct., Cleo)^2 &= 0\\
L.T(Ant., Oct., Cic.)-L.T(Ant., Oct., Cic.)^2 &= 0\\
L.T(Ant., Oct., Cae.)-L.T(Ant., Oct., Cae.)^2 &= 0\\
L.T(Ant., Cleo, Cic.)-L.T(Ant.,  Cleo.,Cic.)^2 &= 0\\
L.T(Ant., Cleo, Cae.)-L.T(Ant.,  Cleo.,Cae.)^2 &= 0\\
L.T(Ant., Cic, Cae.)-L.T(Ant.,  Cic.,Cae.)^2 &= 0\\
L.T(Oct.,Cleo.,Cic.)-L.T(Oct.,Cleo.,Cic.)^2 &= 0\\
L.T(Oct.,Cleo.,Cae.)-L.T(Oct.,Cleo.,Cae.)^2 &= 0\\
L.T(Cleo.,Cae.,Cic.)-L.T(Cleo.,Cae.,Cic.)^2 &= 0\\
L.T(Oct.,Cae.,Cic.)-L.T(Oct.,Cae., Cic.)^2 &= 0\\
e(L.T.(Oct., Ant., Cleo.)) -1 &\geq 0\\
e(L.T.(Cae., Ant., Cleo.)) -1 &\geq 0\\
e(L.T.(Cic., Ant., Cleo.)) -1 &\geq 0\\
e(War(Oct., Ant.) L.T.(Oct., Ant., Cleo.))\qquad&\\
-.75 e(L.T.(Oct., Ant., Cleo.)) &\geq 0\\
e(War(Oct., Ant.) L.T.(Oct., Ant., Cic.))\qquad&\\
-.75 e(L.T.(Oct., Ant., Cic.)) &\geq 0\\
e(War(Oct., Ant.) L.T.(Oct., Ant., Cae.))\qquad&\\
-.75 e(L.T.(Oct., Ant., Cae.)) &\geq 0\\
e(War(Cae., Ant.) L.T.(Cae., Ant., Cleo.))\qquad&\\
-.75 e(L.T.(Cae., Ant., Cleo.)) &\geq 0\\
e(War(Cae., Ant.) L.T.(Cae., Ant., Cic.))\qquad&\\
-.75 e(L.T.(Cae., Ant., Cic.)) &\geq 0\\
e(War(Cae., Ant.) L.T.(Cae., Ant., Oct.))\qquad&\\
-.75 e(L.T.(Cae., Ant., Oct.)) &\geq 0\\
e(War(Cic., Ant.) L.T.(Cic., Ant., Cleo.))\qquad&\\
-.75 e(L.T.(Cic., Ant., Cleo.)) &\geq 0\\
e(War(Cic., Ant.) L.T.(Cic., Ant., Oct.))\qquad&\\
-.75 e(L.T.(Cic., Ant., Oct.)) &\geq 0\\
e(War(Cic., Ant.) L.T.(Cic., Ant., Cae.))\qquad&\\
-.75 e(L.T.(Cic., Ant., Cae.)) &\geq 0
\end{align*}
\begin{align*}
e(War(Cleo., Ant.) L.T.(Cleo., Ant., Oct.))\qquad&\\
-.75 e(L.T.(Cleo., Ant., Oct.)) &\geq 0\\
e(War(Cleo., Ant.) L.T.(Cleo., Ant., Cic.))\qquad&\\
-.75 e(L.T.(Cleo., Ant.,Cic.)) &\geq 0\\
e(War(Cleo., Ant.) L.T.(Cleo., Ant., Cae.))\qquad&\\
-.75 e(L.T.(Cleo., Ant., Cae.)) &\geq 0\\
e(War(Cleo., Oct.) L.T.(Cleo., Oct., Ant.))\qquad&\\
-.75 e(L.T.(Cleo., Oct., Ant.)) &\geq 0\\
e(War(Cleo., Oct.) L.T.(Cleo., Oct., Cic.))\qquad&\\
-.75 e(L.T.(Cleo., Oct.,Cic.)) &\geq 0\\
e(War(Cleo., Oct.) L.T.(Cleo., Oct., Cae.))\qquad&\\
-.75 e(L.T.(Cleo., Oct., Cae.)) &\geq 0\\
e(War(Cleo., Cic.) L.T.(Cleo., Cic., Ant.))\qquad&\\
-.75 e(L.T.(Cleo., Cic., Ant.)) &\geq 0\\
e(War(Cleo., Cic.) L.T.(Cleo., Cic., Oct.))\qquad&\\
-.75 e(L.T.(Cleo., Cic.,Oct.)) &\geq 0\\
e(War(Cleo., Cic.) L.T.(Cleo., Cic., Cae.))\qquad&\\
-.75 e(L.T.(Cleo., Cic., Cae.)) &\geq 0\\
e(War(Cleo., Cae.) L.T.(Cleo., Cae., Ant.))\qquad&\\
-.75 e(L.T.(Cleo., Cae., Ant.)) &\geq 0\\
e(War(Cleo., Cae.) L.T.(Cleo., Cae., Oct.))\qquad&\\
-.75 e(L.T.(Cleo., Cae.,Oct.)) &\geq 0\\
e(War(Cleo., Cae.) L.T.(Cleo., Cae., Cic.))\qquad&\\
-.75 e(L.T.(Cleo., Cae., Cic.)) &\geq 0\\
e(War(Cic., Cae.) L.T.(Cic., Cae., Ant.))\qquad&\\
-.75 e(L.T.(Cic., Cae., Ant.)) &\geq 0\\
e(War(Cic., Cae.) L.T.(Cic., Cae., Oct.))\qquad&\\
-.75 e(L.T.(Cic., Cae.,Oct.)) &\geq 0\\
e(War(Cic., Cae.) L.T.(Cic., Cae., Cleo.))\qquad&\\
-.75 e(L.T.(Cic., Cae., Cleo.)) &\geq 0\\
e(War(Oct., Cae.) L.T.(Oct., Cae., Ant.))\qquad&\\
-.75 e(L.T.(Oct., Cae., Ant.)) &\geq 0\\
e(War(Oct., Cae.) L.T.(Oct., Cae., Cic.))\qquad&\\
-.75 e(L.T.(Oct., Cae.,Cic.)) &\geq 0\\
e(War(Oct., Cae.) L.T.(Oct., Cae., Cleo.))\qquad&\\
-.75 e(L.T.(Oct., Cae., Cleo.)) &\geq 0\\
e(War(Oct., Cic.) L.T.(Oct., Cic., Ant.))\qquad&\\
-.75 e(L.T.(Oct., Cic., Ant.)) &\geq 0\\
e(War(Oct., Cic.) L.T.(Oct., Cic., Cae.))\qquad&\\
-.75 e(L.T.(Oct., Cic.,Cae.)) &\geq 0\\
e(War(Oct., Cic.) L.T.(Oct., Cic., Cleo.))\qquad&\\
-.75 e(L.T.(Oct., Cic., Cleo.)) &\geq 0
\end{align*}

Permutations are omitted.

\section{Omitted Steps of Example~\ref{em:sos}  }

\label{appendix:steps}

Starting with an expression that takes the form of the law of total probability $\forall x \neq Ant., x \neq Oct. e(War(Ant.,Oct.)L.T.(Ant., Oct., x))$... We can then make substitutions using the laws of expectation and probability. These substitutions can be continued until giving a final sum-of-squares expression, which can be simplified to $e(War(Ant., Oct.)) - .75$, of:\\ 

$W^2(1-LT)^2+(1-LT)^2(W-W^2)+W(LT-LT^2)+WLT-.75LT+.75(LT-1)$\\

Where $W = e(War(Ant., Oct.))$, and $e(LT = L.T(Ant., Oct., Cleo.))$. \\

This is an expression of degree 4, meaning that by invoking Theorem~\ref{thm:sossound} we have a degree-4 sum-of-squares program that will detect infeasibility with the constraint $e(War(Ant., Oct.) \leq .75 - \delta$ for any $\delta > 0$. \\ 

First of all, we verify that indeed our expression is equivalent to the query $War(Antony, Octavian)$:

\vspace{0.5mm}
$W^2(1-LT)^2+(1-LT)^2(W-W^2)+W^2(LT-LT^2)
+W\cdot LT-.75LT+.75(LT-1)\\
=W^2((1-LT)^2-(1-LT)^2)+W((1-LT)^2+(LT-LT^2)+LT)-.75\\
=W-.75,$
\vspace{0.3mm}

where $W = War(Antony, Octavian)$, and $LT = LoveTriangle(Antony, Octavian, Cleopatra)$. \\
Since every term is indeed nonnegative by either the knowledge base or because it is a square, we can use this as our sum-of-squares equation to bound the likelihood of a war between Antony and Octavian ($War(Antony, Octavian)$).

Next, we will show how this equation was constructed. Keeping the same values for $W$ and $LT$ we start with the square polynomial:
$W^2(1-LT)^2 = W^2(1-LT)^2$
which, recall, represents $W^2\geq W^2LT^2$.
Then using the Boolean axioms we can reduce the degree of $W$ and $LT$. First $LT$:
\begin{align*}
    W^2(1-LT)^2 + W^2(LT-LT^2) &= W^2(1-LT)
\end{align*}
and next, we will make a substitution of $W$ in for $W^2$:

\vspace{0.5mm}
$W^2(1-L)^2 + W^2(LT-LT^2) + (1-LT)(W-W^2) = W(1-LT)$
\vspace{0.3mm}

Now that we have $W\geq W\cdot LT$, we can add the expression $W\cdot LT-.75LT$ to obtain $W\geq .75LT$:\\
\vspace{0.5mm}
$W^2(1-L)^2 + W^2(LT-LT^2) + (1-LT)(W-W^2) + (W\cdot LT-.75LT) = W - .75LT$

For the final step we will use our given bound for $LT \geq 1$ which can be re-scaled to $.75LT \geq .75$, giving us the expression $.75LT-.75$ for us to add:

$ W^2(1-L)^2 + W^2(LT-LT^2) + (1-LT)(W-W^2) + (W\cdot LT-.75LT) + (.75LT - .75) = W - .75.$

\section{Expectation Logics Example} \label{appen:expecation}
Given a population where 20\% experience an elevated heart rate (100BPM or above) and every person has a heart rate of at least 60, what can we say about the average heart rate of that population:
\begin{align*}
    e(HR) = e(HR * HighHR) + e(HR * (1 - HighHR))
\end{align*}

Now, using the distributivity axiom, and the fact that every person with a high HR has a HR over 100, we can infer:

\begin{align*}
    e(HR * HighHR) \geq e(100 * HighHR)
\end{align*}

We can then do the same with the standard population, using the fact that a heart rate must be above 60, and infer:

\begin{align*}
    e(HR * (1 - HighHR)) \geq e(60 * (1 - HighHR))
\end{align*}

Substituting into the original equation we get:

\begin{align*}
    e(HR) \geq e(100 * HighHR) + e(60 * (1 - HighHR)) \\
    e(HR) \geq 100e(HighHR) + 60e(1 - HighHR) \\
    e(HR) \geq 100 * .2 + 60 * .8\\
    e(HR) \geq 68
\end{align*}

This means that in the overall population the average heart rate is at least 68.

\end{document}